\documentclass[runningheads]{llncs}

\usepackage{eccv}

\usepackage{eccvabbrv}

\usepackage[T1]{fontenc}
\usepackage[utf8]{inputenc}
\usepackage{graphicx}
\usepackage{comment}
\usepackage{url}
\usepackage{booktabs}
\usepackage{amsfonts}
\usepackage{nicefrac}
\usepackage{microtype}
\usepackage{xcolor}

\usepackage{amsmath,amssymb}

\usepackage{amsthm}
\usepackage{colortbl}
\usepackage{multirow}
\usepackage{subcaption}
\usepackage{wrapfig}
\usepackage{xspace}
\usepackage{pifont}
\usepackage{svg}

\newcommand{\seenoevil}{\raisebox{-0.2em}{\includegraphics[height=1.1em]{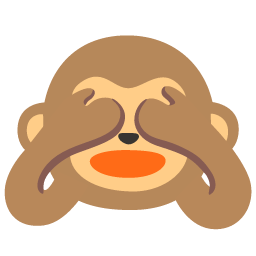}}}
\newcommand{\hearnoevil}{\raisebox{-0.2em}{\includegraphics[height=1.1em]{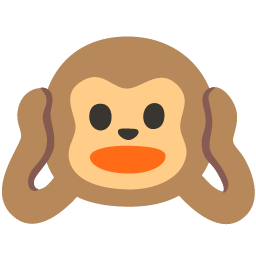}}}
\newcommand{\speaknoevil}{\raisebox{-0.2em}{\includegraphics[height=1.1em]{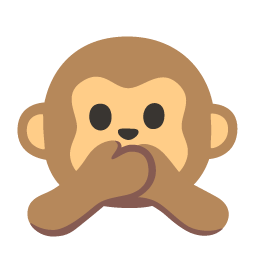}}}
\newcommand{\monkey}{\raisebox{-0.2em}{\includegraphics[height=1.3em]{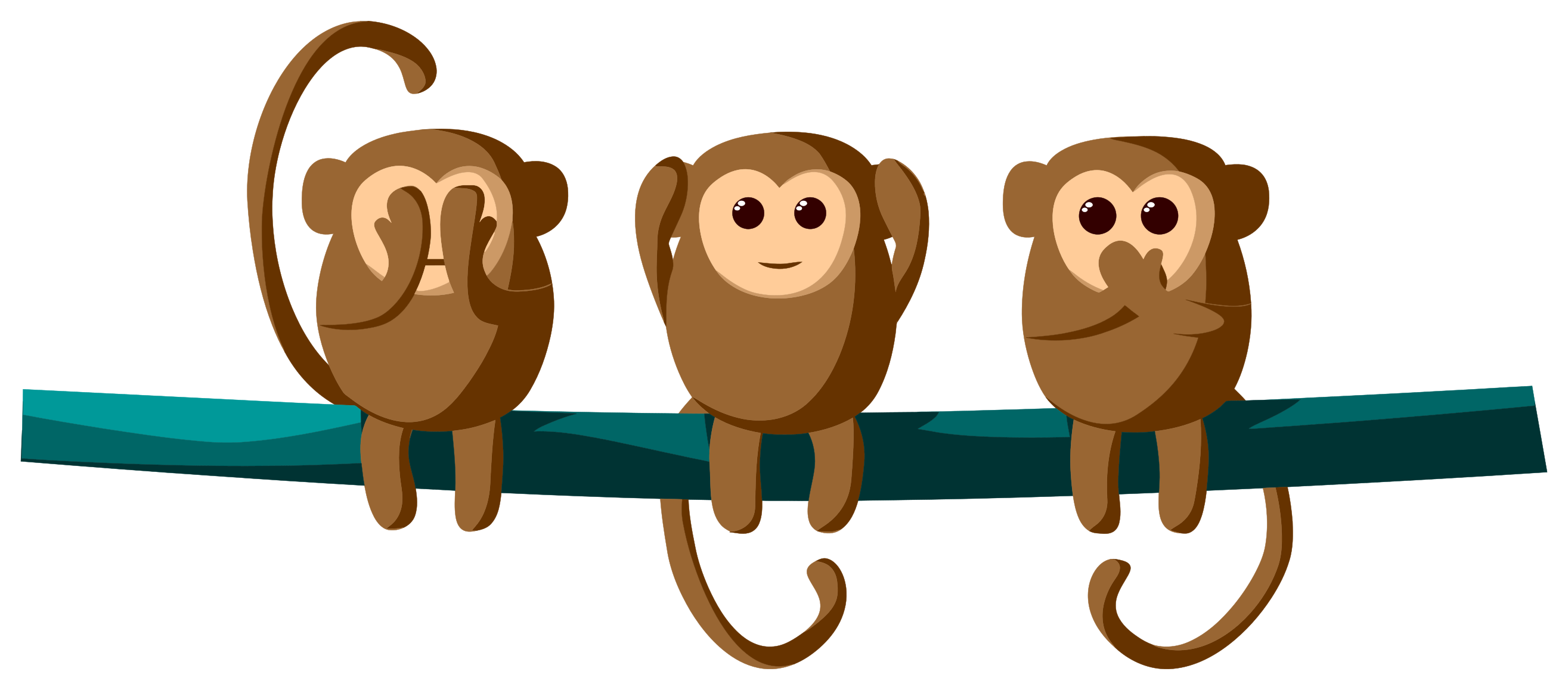}}}

\newcommand{\xmark}{\ding{55}}
\newcommand{\cmark}{\checkmark}
\newcommand{\bx}{\mathbf{x}}
\newcommand{\by}{\mathbf{y}}
\newcommand{\bz}{\mathbf{z}}
\newcommand{\bs}{\mathbf{s}}
\newcommand{\sx}{\mathbf{s}_x}
\newcommand{\sy}{\mathbf{s}_y}
\newcommand{\sg}{\mathrm{sg}}
\newcommand{\Fproj}{F^{\mathrm{proj}}_w}
\newcommand{\Epred}{E^{\mathrm{pred}}_w}
\newcommand{\Fpred}{F^{\mathrm{pred}}_w}
\newcommand{\Linv}{\mathcal{L}_{\mathrm{inv}}}
\newcommand{\Lpred}{\mathcal{L}_{\mathrm{pred}}}
\newcommand{\Lreg}{\mathcal{L}_{\mathrm{reg}}}
\newcommand{\Ltotal}{\mathcal{L}_{\mathrm{total}}}
\newcommand{\syn}[2]{\mathbf{s}_{y,#1}^{(#2)}}

\theoremstyle{plain}
\newtheorem{thmx}[equation]{Theorem}
\newtheorem{defx}[equation]{Definition}
\newtheorem{propx}[equation]{Proposition}
\newtheorem{lemx}[equation]{Lemma}
\newtheorem{corx}[equation]{Corollary}
\newtheorem{remx}[equation]{Remark}

\newenvironment{definition}[2]{\begin{defx}[#1]\label{def:#2}}{\end{defx}}
\newenvironment{theorem}[2]{\begin{thmx}[#1]\label{theorem:#2}}{\end{thmx}}
\newenvironment{proposition}[2]{\begin{propx}[#1]\label{proposition:#2}}{\end{propx}}
\newenvironment{lemma}[2]{\begin{lemx}[#1]\label{lemma:#2}}{\end{lemx}}
\newenvironment{corollary}[2]{\begin{corx}[#1]\label{corollary:#2}}{\end{corx}}
\newenvironment{remark}[2]{\begin{remx}[#1]\label{remark:#2}}{\end{remx}}

\usepackage{hyperref}

\usepackage[accsupp]{axessibility}  

\usepackage[capitalize,nameinlink]{cleveref}

\crefname{thmx}{Theorem}{Theorems}
\Crefname{thmx}{Theorem}{Theorems}
\crefname{defx}{Definition}{Definitions}
\Crefname{defx}{Definition}{Definitions}
\crefname{propx}{Proposition}{Propositions}
\Crefname{propx}{Proposition}{Propositions}
\crefname{lemx}{Lemma}{Lemmas}
\Crefname{lemx}{Lemma}{Lemmas}
\crefname{corx}{Corollary}{Corollaries}
\Crefname{corx}{Corollary}{Corollaries}
\crefname{remx}{Remark}{Remarks}
\Crefname{remx}{Remark}{Remarks}

\begin{document}

\title{\texorpdfstring{\monkey}{} Three Necessary Principles\\for Self-Supervised Visual Representation Learning}

\titlerunning{Three Necessary Principles for SSL}

\author{
Nikos Giakoumoglou\inst{1} \and
Paschalis Giakoumoglou\inst{2} \and
Tania Stathaki\inst{1}
}

\authorrunning{N. Giakoumoglou et al.}

\institute{Imperial College London\\
\email{\{nikos,tania\}@imperial.ac.uk} \and
CERTH, ITI\\
\email{giakoupg@iti.gr}
}

\maketitle


\begin{abstract}

We argue that learning visual representations without labels requires a training signal jointly complete across three non-overlapping objectives: semantic invariance across augmented views, patch-level spatial prediction, and representational non-degeneracy. We formalize these as the \emph{observation}, \emph{prediction}, and \emph{regularization} principles and prove \textbf{(i)}~that combining observation and prediction without regularization admits the constant encoder as a global minimizer under negative-free alignment; \textbf{(ii)}~that the two objectives are gradient-complementary and structurally non-conflicting at the encoder output; and \textbf{(iii)}~that the momentum encoder converges to the same fixed point as the online encoder and provides no collapse guarantee at convergence. Contrastive alignment provides only self-limiting collapse resistance, formalized via an explicit gradient-decay argument. Dropping prediction withholds the spatial training signal by construction; dropping observation forfeits cross-view semantic invariance by construction; at the scale we study, no pair substitutes for the third. Every major self-supervised method is a special case of a single unified energy decomposition. We pair every theoretical claim with a controlled experiment, including a patch-retrieval evaluation for the spatial consequence of prediction.


\end{abstract}

\section{Introduction}
\label{sec:introduction}

Self-supervised visual representation learning can be understood through a single lens: methods learn a compatibility function over pairs of observations, assigning low energy to compatible inputs and high energy otherwise~\cite{lecun2022path,dawid2023introduction}. The central design question is what structure the energy surface should encode, and how. In practice this reduces to two components: a projector, which maps globally pooled representations to an embedding space where view-level compatibility is measured, and a predictor, which maps context patches and positional information to predicted target representations. \emph{Invariance} methods~\cite{oord2019cpc,he2020momentum,chen2020simclr,grill2020byol,chen2021simsiam,bardes2022vicreg,zbontar2021barlow,giakoumoglou2024synco} train only the projector, mapping augmented views of the same image to nearby representations; the result is semantically rich but discards spatial structure through global pooling. \emph{Joint-Embedding Predictive Architectures} (JEPAs)~\cite{lecun2022path,assran2023ijepa,bardes2024vjepa} train only the predictor, recovering the latent representation of masked targets from visible context; the abstract target lets the encoder discard irrelevant variation, but without an invariance objective the representations can be locally consistent yet globally unstructured, lacking the cross-view semantic alignment that invariance methods provide. Both families face a common failure mode, \emph{representational collapse}~\cite{jing2021understanding}, each handled with its own mechanism, negative mining, asymmetric dynamics, or statistical regularization, typically conflated with the primary objective rather than treated as a separate requirement.

We distinguish three non-equivalent failure modes. \emph{Optimization collapse}: the constant encoder achieves zero training loss (Theorem~\ref{theorem:collapse}). \emph{Dimensional collapse}: the representation covariance has low effective rank, measurable even at non-zero loss. \emph{Semantic impoverishment}: low linear-probe accuracy despite non-trivial effective rank. The three can occur independently (e.g., Table~\ref{tab:components} rows~\textcolor{gray}{H} and~\textcolor{gray}{I}). Definition~\ref{def:reg} addresses dimensional collapse; optimization collapse is a sufficient condition for it when no other mechanism is present.

We argue that a complete training signal must satisfy three distinct and non-overlapping conditions. An encoder must \emph{observe}~\seenoevil\ (Definition~\ref{def:obs}: produce semantically invariant global representations via the projector), \emph{predict}~\hearnoevil\ (Definition~\ref{def:pred}: support latent-space prediction of masked content via the predictor), and \emph{regularize}~\speaknoevil\ (Definition~\ref{def:reg}: prevent dimensional collapse via an explicit geometric constraint). We prove regularization necessary in the negative-free regime (Theorem~\ref{theorem:collapse}) and formalize the self-limiting contrastive resistance via a gradient-decay argument (Remark~\ref{remark:ntxent_partial}). For prediction and observation we give structural arguments and validate empirically, including with a patch-retrieval evaluation (Section~\ref{sec:experiments}). We further prove observation and prediction gradient-complementary at the encoder output (Theorem~\ref{theorem:gradient_decomp}, for both $\mathcal{L}_{\mathrm{MSE}}$ and $\mathcal{L}_{\mathrm{NT\text{-}Xent}}$ under a mean-pool projector).

This decomposition unifies the field~\cite{giakoumoglou2024review}; every major method is recovered by zeroing coefficients of a single energy. Removing prediction recovers invariance methods~\cite{chen2020simclr,he2020momentum,grill2020byol,chen2021simsiam,zbontar2021barlow,bardes2022vicreg,giakoumoglou2024synco}; removing observation recovers predictive methods~\cite{assran2023ijepa,balestriero2025lejepa,kuang2026lpjepa}. Methods satisfying all three but retaining implicit mechanisms~\cite{mo2024cjepa,oquab2024dinov2} exhibit a redundancy we prove formally (Theorem~\ref{theorem:ema_collapse}) and confirm empirically (Table~\ref{tab:momentum}).

\noindent\textbf{Contributions.} \textbf{(i)}~We formalize the three principles as non-overlapping conditions on the energy surface (Section~\ref{sec:conditions}), prove regularization optimization-theoretically necessary under negative-free alignment, formalize contrastive self-limiting resistance via a gradient-decay argument, and argue prediction and observation structurally necessary. \textbf{(ii)}~We show all major SSL methods are special cases of a unified energy decomposition (Table~\ref{tab:taxonomy}). \textbf{(iii)}~We prove gradient complementarity of observation and prediction (Theorem~\ref{theorem:gradient_decomp}) and momentum-encoder redundancy at convergence (Theorem~\ref{theorem:ema_collapse}). \textbf{(iv)}~We pair every theoretical claim with a controlled experiment on ViT-Tiny/STL-10, including patch-retrieval for the spatial consequence of prediction (Section~\ref{sec:experiments}).

\section{Background}
\label{sec:background}

We develop the minimal Energy-Based Model (EBM) formalism our argument needs; the standard contrastive and regularized training strategies and their energy landscapes are recalled in Appendix~\ref{appendix:ebm}.

Following LeCun \etal~\cite{lecun2022path} and Dawid \etal~\cite{dawid2023introduction}, an EBM is a scalar function $F_w : \mathcal{X} \times \mathcal{Y} \rightarrow \mathbb{R}$ that is low when $\bx$ and $\by$ are \emph{compatible} and high otherwise. In self-supervised vision, $\bx$ and $\by$ are two augmented views of the same image, or a visible context and a masked target. A latent-variable generalization introduces $\bz \in \mathcal{Z}$ and eliminates it by minimization, yielding the free energy $F_w(\bx, \by) = \min_{\bz \in \mathcal{Z}} E_w(\bx, \by, \bz)$~\cite{lecun2022path}. The difficulty is preventing the low-energy region from expanding to cover all of $\mathcal{Y}$; contrastive methods raise the energy of negatives, while regularized methods constrain the low-energy volume directly (Appendix~\ref{appendix:ebm}). The degenerate failure is \emph{representational collapse}, $f_\theta(\bx) = c$ for all $\bx$, or more generally a low-dimensional subspace~\cite{jing2021understanding}.

The JEPA~\cite{lecun2022path,assran2023ijepa} instantiates this framework with a shared encoder $f_\theta$ and predictor $g_\phi$: $E_w(\bx, \by, \bz) = D(f_\theta(\by), g_\phi(f_\theta(\bx), \bz))$. The critical distinction from pixel-space reconstruction~\cite{he2022mae,baevski2022data2vec} is that the target is the representation $f_\theta(\by)$ rather than the raw signal, so the encoder may discard irrelevant variation and the predictor need only recover abstract spatial structure~\cite{lecun2022path}.

\paragraph{\bf Energy surface realizations.} Two architectural components realize this energy in practice. A \emph{projector} $h_\psi : \mathbb{R}^D \rightarrow \mathbb{R}^{D'}$ maps globally pooled representations to a space where view-level compatibility is measured, inducing
\begin{equation}
    \Fproj(\bx, \by) = d\!\left(h_\psi(\sx),\, h_\psi(\sy)\right),
    \label{eq:Fproj}
\end{equation}
where $\sx, \sy$ are the mean-pooled representations and $d$ is a distance. A \emph{predictor} $g_\phi : \mathbb{R}^{|\mathcal{C}| \times D} \times \mathcal{Z} \rightarrow \mathbb{R}^D$ maps context patches and positional information to predicted target representations, inducing
\begin{equation}
\begin{aligned}
\Epred(\bx, \by, \bz)
    &= D\!\left(f_\theta(\by)^{(t)},\;
        g_\phi(f_\theta(\bx)|_\mathcal{C},\, \bz)\right), \\
\Fpred(\bx, \by)
    &= \min_{\bz} \Epred(\bx, \by, \bz).
\end{aligned}
\label{eq:Epred}
\end{equation}
Thus $\Fproj$ shapes the surface globally (what content is present, irrespective of where) while $\Fpred$ shapes it locally (where content is, and how it relates spatially). As we formalize next, these aspects are non-overlapping and both necessary.

\section{Three Necessary Principles}
\label{sec:conditions}

Self-supervised methods differ not in their underlying principle, since all shape an energy surface to reflect compatibility, but in which aspect they model and how they prevent collapse. A complete training signal must satisfy three non-overlapping conditions, which we obtain by decomposing the total energy as
\begin{equation}
    F_w(\bx, \by) = \alpha\, \Fproj(\bx, \by) \;+\; \beta\, \Fpred(\bx, \by) \;+\; \gamma\, \Omega(f_\theta),
    \label{eq:unified}
\end{equation}
with scalar weights $\alpha, \beta, \gamma \geq 0$ and a geometric regularizer $\Omega$ on the batch distribution of encoder representations. Setting individual coefficients to zero recovers strict subsets of the three conditions.

\begin{definition}{\seenoevil~Observation Principle}{obs}
An encoder $f_\theta$ with projector $h_\psi$ satisfies the \emph{observation principle} if, for any two compatible views, $\Fproj(\bx, \by) = d(h_\psi(\sx), h_\psi(\sy))$ is minimized; the projector shapes the energy surface so that same-image view pairs receive low energy irrespective of augmentation.
\end{definition}

Because $d$ operates on global representations, spatial arrangement within $\bx$ is not preserved; the observation principle is necessary but not sufficient for a world model.

\begin{definition}{\hearnoevil~Prediction Principle}{pred}
An encoder $f_\theta$ with predictor $g_\phi$ satisfies the \emph{prediction principle} if, given context $\bx$ and target $\by$ with positional information $m_t$, $\Epred(\bx, \by, m_t) = D(f_\theta(\by)^{(t)}, g_\phi(f_\theta(\bx)|_\mathcal{C}, m_t))$ is minimized.
\end{definition}

The prediction target $f_\theta(\by)^{(t)}$ is a latent representation rather than the raw signal, so the encoder may discard irrelevant variation. Satisfying this within an image does not enforce that $\sx$ and $\sy$ be close across augmented views; prediction does not imply observation.

\begin{definition}{\speaknoevil~Regularization Principle}{reg}
An encoder $f_\theta$ satisfies the \emph{regularization principle} if the training objective contains a term $\Omega$ depending only on the empirical batch distribution $\hat{p} = \frac{1}{B}\sum_{n=1}^{B} \delta_{\bs_{y,n}}$, strictly positive at every Dirac mass and minimized only at distributions whose representation covariance has full effective rank over $\mathbb{R}^D$.
\end{definition}

This places no constraint on the content of representations, only their geometry; $\Omega$ shrink-wraps the low-energy region around the data manifold, preventing both complete and dimensional collapse~\cite{jing2021understanding}. The three instantiations we study (Section~\ref{sec:method}, Appendix~\ref{appendix:reg_prevents}) each meet this definition.

\subsection{Necessity}
\label{subsec:necessity}

We now argue that each of the three principles is individually necessary; dropping any one admits a degenerate or impoverished solution that the remaining two cannot rule out at the scale we study. For regularization we give an optimization-theoretic proof under negative-free alignment and formalize contrastive resistance via a gradient-decay argument. For prediction and observation the arguments are structural: the omitted signal is simply never provided.

\begin{proposition}{Necessity of Regularization}{necessity_reg}
Satisfying Definitions~\ref{def:obs} and~\ref{def:pred} without Definition~\ref{def:reg} admits dimensional collapse at the studied scale. Under $\Linv = \mathcal{L}_{\mathrm{MSE}}$, the constant encoder is a global minimizer achieving zero loss (Theorem~\ref{theorem:collapse}). Under $\Linv = \mathcal{L}_{\mathrm{NT\text{-}Xent}}$, the contrastive term acts as an implicit regularizer whose resistance is batch-size-dependent and self-limiting: as representations homogenize, all pairwise similarities converge and the NT-Xent gradient with respect to the encoder approaches zero (Remark~\ref{remark:ntxent_partial}), so the contrastive term cannot substitute for an explicit $\Lreg$ that maintains a positive gradient at every step.
\end{proposition}

\begin{theorem}{Collapse under observation and prediction without regularization}{collapse}
Let $\gamma = 0$ and $\Linv = \mathcal{L}_{\mathrm{MSE}}$. Then the constant encoder $f_\theta(\bx) = \mathbf{c}\,\mathbf{1}_N^\top$ for all $\bx$, with a predictor satisfying $g_\phi(\mathbf{c}\,\mathbf{1}_{|\mathcal{C}|}^\top, m_t) = \mathbf{c}$, is a global minimizer of $\alpha\mathcal{L}_{\mathrm{MSE}} + \beta\Lpred$ achieving value zero, for any $\alpha, \beta > 0$.
\end{theorem}
\begin{proof}[Proof sketch]
With $f_\theta$ constant both views map to the same point, so $\mathcal{L}_{\mathrm{MSE}} = 0$; a predictor with output bias $\mathbf{c}$ drives $\Lpred = 0$, attainable through the context branch alone since the stop-gradient blocks the target branch. Both terms being non-negative and zero, the solution is a global minimizer. Full proof in Appendix~\ref{appendix:collapse}.
\end{proof}

\begin{remark}{Self-limiting collapse resistance of the contrastive loss}{ntxent_partial}
Theorem~\ref{theorem:collapse} does not extend to $\Linv = \mathcal{L}_{\mathrm{NT\text{-}Xent}}$; with in-batch negatives the constant encoder is not a global minimizer. However, as representations homogenize toward a common direction, all pairwise cosine similarities converge to 1, the softmax denominator saturates at $K \cdot e^{1/\tau}$, and the gradient contribution from each negative shrinks as $\mathcal{O}(1/K)$. The total repulsive signal therefore scales as $\mathcal{O}(1)$ in the fully-collapsed limit rather than growing to prevent it; the contrastive term cannot guarantee a positive gradient sufficient to prevent gradual dimensional collapse at every step, unlike an explicit $\Lreg$ whose gradient is strictly positive at every Dirac mass (Appendix~\ref{appendix:reg_prevents}). Intuitively, once all negatives are equally hard the contrastive loss has no preferred direction in which to push representations apart, and the implicit anti-collapse signal stalls. Empirically, Table~\ref{tab:components} row~\textcolor{gray}{H} retains $51.6\%$ without explicit regularization but never matches row~\textcolor{gray}{J} ($55.0\%$) with it, and the effective rank in row~\textcolor{gray}{H} is batch-size-sensitive (Appendix~\ref{appendix:batchsize}).
\end{remark}

\begin{proposition}{Structural necessity of prediction}{necessity_pred}
Under the studied protocol, satisfying Definitions~\ref{def:obs} and~\ref{def:reg} without Definition~\ref{def:pred} yields representations that encode semantic content but receive no spatial training signal by construction: no term in $\alpha\Fproj + \gamma\Omega$ depends on the relative positions of patches, so the encoder is never required to make $g_\phi(f_\theta(\bx)|_\mathcal{C}, m_t)$ approximate $f_\theta(\by)^{(t)}$. This structural omission manifests as degraded patch retrieval accuracy (Table~\ref{tab:retrieval}) and lower linear-probe accuracy (Table~\ref{tab:components}, row~\textcolor{gray}{F} vs.\ row~\textcolor{gray}{J}).
\end{proposition}

\begin{proposition}{Structural necessity of observation}{necessity_obs}
Under the studied protocol, satisfying Definitions~\ref{def:pred} and~\ref{def:reg} without Definition~\ref{def:obs} yields representations that are spatially consistent within an image but receive no cross-view semantic alignment signal by construction: no term in $\beta\Fpred + \gamma\Omega$ enforces that $\sx$ and $\sy$ be close across augmented views. This structural omission manifests as the largest single accuracy drop in the ablation (Table~\ref{tab:components}, row~\textcolor{gray}{E} vs.\ row~\textcolor{gray}{J}, a gap of $13.4$ points).
\end{proposition}

Propositions~\ref{proposition:necessity_pred} and~\ref{proposition:necessity_obs} are structural, not optimization-theoretic: they assert the omitted signal is never provided, not that no encoder could acquire the property indirectly. The alignment-uniformity decomposition of Wang \etal~\cite{wang2020understanding} shows the observation and regularization roles are entangled within a single contrastive loss. Our empirical validation is at a single small scale (Section~\ref{sec:experiments}); whether a larger encoder could acquire spatial structure from observation alone, or cross-view invariance from prediction alone, remains open.

\subsection{Relation to Prior Methods}
\label{subsec:taxonomy}

The same decomposition that defines the three principles also classifies prior work; each method is recovered by fixing which coefficients of \Cref{eq:unified} are non-zero.

\begin{proposition}{Unification of prior methods}{unification}
The decomposition of \Cref{eq:unified} recovers all major self-supervised methods as special cases. Setting $\beta = 0$ recovers invariance methods, with collapse prevented contrastively~\cite{oord2019cpc,he2020momentum,chen2020simclr}, through asymmetric dynamics~\cite{grill2020byol,chen2021simsiam}, or an explicit regularizer~\cite{bardes2022vicreg,zbontar2021barlow}. Setting $\alpha = 0$ recovers predictive methods: I-JEPA~\cite{assran2023ijepa} with a momentum encoder; LeJEPA~\cite{balestriero2025lejepa} and LpJEPA~\cite{kuang2026lpjepa} with explicit regularizers. With all three positive, C-JEPA~\cite{mo2024cjepa} and DINOv2~\cite{oquab2024dinov2}, whose KoLeo term is an explicit geometric regularizer, satisfy the three conditions but retain implicit mechanisms (momentum encoders, centering) alongside the regularizer. We prove this redundancy formally at convergence (Theorem~\ref{theorem:ema_collapse}) and confirm it empirically (Table~\ref{tab:momentum}); our proposed model in Section~\ref{sec:experiments} satisfies all three with no implicit mechanism.
\end{proposition}

The complete taxonomy is given in Table~\ref{tab:taxonomy} (Appendix~\ref{appendix:taxonomy}). Methods that prevent collapse only through implicit mechanisms (momentum encoder, stop-gradient, contrastive repulsion, centering) do not satisfy Definition~\ref{def:reg} and give no guarantee at convergence (Theorem~\ref{theorem:ema_collapse}).

\section{An Energy Decomposition Satisfying All Three Principles}
\label{sec:method}

We instantiate \Cref{eq:unified} as three jointly trained objectives, $\Ltotal = \alpha \Linv + \beta \Lpred + \gamma \Lreg$. Because collapse prevention is delegated solely to $\Lreg$, both $\Linv$ and $\Lpred$ are free to focus on their own inductive biases without being redesigned to guard against degenerate solutions. We write $\sx, \sy \in \mathbb{R}^D$ for the mean-pooled encoder outputs.

\paragraph{\bf Invariance objective.} Instantiating Definition~\ref{def:obs} through $\Fproj$, we study two forms. The normalized temperature-scaled cross-entropy loss~\cite{chen2020simclr} aligns views with in-batch negatives,
\begin{equation}
    \mathcal{L}_{\mathrm{NT\text{-}Xent}} = -\frac{1}{B} \sum_{n=1}^{B} \log \frac{\exp(\mathrm{sim}(h_\psi(\sx^{(n)}), h_\psi(\sy^{(n)})) / \tau)}{\sum_{k \neq n} \exp(\mathrm{sim}(h_\psi(\sx^{(n)}), h_\psi(\sy^{(k)})) / \tau)},
    \label{eq:ntxent}
\end{equation}
with cosine similarity $\mathrm{sim}$ and temperature $\tau$; its negative repulsion is an implicit anti-collapse signal (Remark~\ref{remark:ntxent_partial}). The alternative is a negative-free alignment $\mathcal{L}_{\mathrm{MSE}} = \frac{1}{B}\sum_n \|h_\psi(\sx^{(n)}) - h_\psi(\sy^{(n)})\|_2^2$, which collapses trivially without $\Lreg$ (Theorem~\ref{theorem:collapse}) and serves as a clean diagnostic of each regularizer's sufficiency. Note that ``negative-free'' refers throughout to the absence of contrastive negatives in the invariance objective; it does not refer to prediction-only methods, which have no invariance objective at all.

\paragraph{\bf Predictive objective.} Instantiating Definition~\ref{def:pred} through $\Fpred$, we follow the I-JEPA masking strategy~\cite{assran2023ijepa}: a context block $\mathcal{C}$ is sampled per view and target blocks are masked, and the predictor maps context and mask tokens to target representations,
\begin{equation}
    \Lpred = \frac{1}{B \cdot |\mathcal{T}|} \sum_{n=1}^{B} \sum_{t \in \mathcal{T}} \left\| g_\phi(f_\theta(\bx_n)|_\mathcal{C},\, m_t) - \sg\!\left(\syn{n}{t}\right) \right\|_2^2,
    \label{eq:pred}
\end{equation}
where $\syn{n}{t} = f_\theta(\by_n)^{(t)}$ and $\sg(\cdot)$ is the stop-gradient blocking flow through the target branch, so all collapse prevention is delegated to $\Lreg$.

\paragraph{\bf Regularization objective.} Instantiating Definition~\ref{def:reg}, we study three regularizers on the batch distribution of projected representations: \emph{VCReg}~\cite{bardes2022vicreg} (variance and covariance terms), \emph{SIGReg}~\cite{balestriero2025lejepa} (matching an isotropic Gaussian via the Epps-Pulley statistic), and \emph{RDMReg}~\cite{kuang2026lpjepa} (matching a rectified generalized Gaussian via sliced Wasserstein distance). Each is strictly positive at the collapsed solution and therefore meets Definition~\ref{def:reg} (Appendix~\ref{appendix:reg_prevents}). Full definitions appear in Appendix~\ref{appendix:objectives}.

\subsection{Theoretical Guarantees}
\label{subsec:theory_main}

Two further properties justify the design: the invariance and prediction objectives do not interfere at the encoder output, and the momentum encoder common to prior work supplies no collapse guarantee beyond an explicit regularizer once that regularizer is present.

\begin{theorem}{Gradient complementarity of observation and prediction}{gradient_decomp}
Under a mean-pool projector architecture (where $h_\psi$ acts on the mean-pooled token $\sx$), the gradient of $\Linv$ with respect to each patch token is uniform across all $i \in [N]$: $\partial \Linv / \partial f_\theta(\bx)^{(i)} = \frac{1}{N}\,\partial \Linv / \partial \sx$. This holds for both $\Linv = \mathcal{L}_{\mathrm{MSE}}$ and $\Linv = \mathcal{L}_{\mathrm{NT\text{-}Xent}}$, since both act on patch tokens only through the pooled representation. The gradient of $\Lpred$ is zero for all masked tokens $i \notin \mathcal{C}$. Thus $\Linv$ provides a uniform global signal and $\Lpred$ a spatially local one, with no structural conflict at the encoder output.
\end{theorem}
\begin{proof}[Proof sketch]
Both $\mathcal{L}_{\mathrm{MSE}}$ and $\mathcal{L}_{\mathrm{NT\text{-}Xent}}$ depend on each patch token only through the mean-pooled $\sx$ and the projector $h_\psi$, so $\partial\sx/\partial f_\theta(\bx)^{(i)} = \frac{1}{N}\mathbf{I}_D$ gives the uniform projection for both forms. $\Lpred$ receives only context tokens $f_\theta(\bx)|_\mathcal{C}$ as input to $g_\phi$; masked tokens do not appear in the computational graph, and the stop-gradient on $\syn{n}{t}$ blocks all flow through the target branch. Full proof in Appendix~\ref{appendix:grad}.
\end{proof}

\begin{remark}{Token-level versus parameter-level conflict}{param_conflict}
Theorem~\ref{theorem:gradient_decomp} establishes disjointness at the encoder output under a mean-pool projector: a uniform global field from $\Linv$ and a context-supported local field from $\Lpred$. It does not preclude interference in parameter space, where both fields propagate into the shared weights $\theta$. We read the theorem as ruling out structural conflict at the interface where the two signals are defined; parameter-space alignment remains an empirical question, directly testable by measuring the gradient cosine between the two objectives in $\theta$ across training.
\end{remark}

\begin{theorem}{Momentum encoder redundancy (conditional on convergence)}{ema_collapse}
\emph{Assumption:} the online encoder $\theta(t)$ converges to a fixed point $\theta^*$ under the given training objective. Under this assumption, in the continuous-time gradient flow, the momentum target encoder satisfies $\bar\theta(t) \to \theta^*$. If $\theta^*$ is collapsed (a failure mode that \Cref{theorem:collapse} and Appendix~\ref{appendix:reg_prevents} show cannot arise when $\gamma > 0$), the momentum encoder is also collapsed at convergence and provides no corrective signal. Note: this analysis models a decoupled flow; the full coupled dynamics with stop-gradient are more complex and convergence of $\theta(t)$ is not guaranteed in general.
\end{theorem}
\begin{proof}[Proof sketch]
The average $\dot{\bar\theta} = (1-m)(\theta - \bar\theta)$ is a stable low-pass filter of $\theta(t)$; once $\theta(t) \to \theta^*$ it tracks it and $\bar\theta(t) \to \theta^*$. Full proof in Appendix~\ref{appendix:ema}.
\end{proof}

The explicit $\Lreg$ prevents the collapsed state from being a fixed point for any $\gamma > 0$ (Appendix~\ref{appendix:proofs}), giving a guarantee at every step. Theorem~\ref{theorem:ema_collapse} shows the momentum encoder cannot supply a collapse guarantee beyond the explicit regularizer, but leaves open a transient stabilizing role of the kind analyzed by Tian \etal~\cite{tian2021understanding} for negative-free methods. Table~\ref{tab:momentum} reflects exactly this division of labor: all three target-encoder choices reach a healthy asymptote whose location is set by the regularizer, while the momentum variant retains a modest accuracy edge attributable to smoothing of the prediction targets during training.

\section{Related Work}
\label{sec:related}

\paragraph{\bf Invariance and masked modeling.} Contrastive methods~\cite{oord2019cpc,he2020momentum,chen2020simclr,giakoumoglou2024synco,giakoumoglou2025cluster} and their negative-free~\cite{grill2020byol,chen2021simsiam} and regularized~\cite{bardes2022vicreg,zbontar2021barlow} variants satisfy Definition~\ref{def:obs} but not Definition~\ref{def:pred}, discarding spatial structure through global pooling. Masked image modeling reconstructs masked content in pixel or token space~\cite{bao2021beit,he2022mae,baevski2022data2vec}, training no invariance objective. Self-distillation methods~\cite{caron2021dino,zhou2022ibot,oquab2024dinov2} combine view-level invariance with latent prediction of masked tokens; iBOT relies on implicit collapse prevention, while DINOv2's KoLeo term satisfies Definition~\ref{def:reg} (Table~\ref{tab:taxonomy}). Wang \etal~\cite{wang2020understanding} decompose the contrastive loss into alignment and uniformity, anticipating the observation-regularization split; Garrido \etal~\cite{garrido2023duality} prove a contrastive-covariance duality predicting the regularizer interchangeability in Table~\ref{tab:reg}; Garrido \etal~\cite{garrido2023rankme} establish effective rank as a downstream predictor whose failure modes we also observe.

\paragraph{\bf Joint-Embedding Predictive Architectures.} I-JEPA~\cite{assran2023ijepa} instantiates the predictor energy with a momentum encoder for collapse prevention, extended to video by V-JEPA~\cite{bardes2024vjepa,assran2025vjepa2}. LeJEPA~\cite{balestriero2025lejepa} and LpJEPA~\cite{kuang2026lpjepa} replace the momentum encoder with explicit regularizers, satisfying Definitions~\ref{def:pred} and~\ref{def:reg} but not Definition~\ref{def:obs}. Closest to ours, C-JEPA~\cite{mo2024cjepa} augments I-JEPA with VICReg terms to satisfy all three; relative to it we eliminate the momentum encoder entirely (Theorem~\ref{theorem:ema_collapse}), treat the three objectives as fully modular, and ground the decomposition in the formal principles of Section~\ref{sec:conditions}. Tian \etal~\cite{tian2021understanding} analyze the transient dynamics by which stop-gradient and momentum prevent collapse in negative-free methods; Theorem~\ref{theorem:ema_collapse} is complementary, characterizing the fixed point rather than the trajectory.

\section{Experiments}
\label{sec:experiments}

\paragraph{\bf Setup.} We pre-train a ViT-Tiny encoder (embedding dimension 192, depth 12, patch size 8) on the unlabeled split of STL-10~\cite{coates2011stl10} at $96{\times}96$, giving $144$ patch tokens. The projector maps the pooled representation through width 1024 to dimension 256; the predictor is a narrow ViT (embedding dimension 96, depth 4). Masking follows I-JEPA~\cite{assran2023ijepa}: one context block (scale $[0.85,1.0]$) and four target blocks (scale $[0.15,0.2]$). We train 200 epochs with AdamW, batch size 512, and a cosine-decayed peak learning rate $5{\times}10^{-4}$, evaluating by linear probe and patch-level retrieval (\Cref{subsec:retrieval}). We report the effective rank of the projector-output covariance as a scalar measure of non-degeneracy. All results are means over 5 random seeds; full hyperparameters appear in Appendix~\ref{appendix:hyperparams}.

\begin{table}[t]
    \small
    \centering
    \caption{\textbf{Subset ablation isolating the contribution of each principle.} Each row instantiates a strict subset of $\{\Linv, \Lpred, \Lreg\}$, verifying Propositions~\ref{proposition:necessity_reg}, \ref{proposition:necessity_pred}, and \ref{proposition:necessity_obs} and Theorem~\ref{theorem:collapse} and Remark~\ref{remark:ntxent_partial}. Our proposed model (row~\textcolor{gray}{J}) is the full contrastive configuration and serves as the reference. Row~\textcolor{gray}{G} (MSE + JEPA, no regularizer) collapses to low effective rank despite satisfying two principles, directly verifying Theorem~\ref{theorem:collapse}; the low effective rank of row~\textcolor{gray}{I} despite explicit regularization is discussed in Section~\ref{subsec:mse_anomaly}.}
    \label{tab:components}
    \begin{tabular}{clllcc}
        \toprule
         & $\Linv$ & $\Lpred$ & $\Lreg$ & Effective Rank & Linear Probe \\
        \midrule
        \textcolor{gray}{A} & MSE     & none & none    & 20.0  & 17.3 \\
        \textcolor{gray}{B} & NT-Xent & none & none    & 166.0 & 41.8 \\
        \textcolor{gray}{C} & none    & JEPA & none    & 15.9  & 42.1 \\
        \textcolor{gray}{D} & none    & none & SIGReg  & 7.3   & 21.2 \\
        \midrule
        \textcolor{gray}{E} & none    & JEPA & SIGReg  & 54.3  & 41.6 \\
        \textcolor{gray}{F} & NT-Xent & none & SIGReg  & 89.3  & 49.5 \\
        \textcolor{gray}{G} & MSE     & JEPA & none    & 164.5 & 38.7 \\
        \textcolor{gray}{H} & NT-Xent & JEPA & none    & 167.7 & 51.6 \\
        \textcolor{gray}{I} & MSE     & JEPA & SIGReg  & 9.4   & 49.2 \\
        \rowcolor{gray!12}
        \textcolor{gray}{J} & NT-Xent & JEPA & SIGReg  & 87.5  & \textbf{55.0} \\
        \bottomrule
    \end{tabular}
\end{table}

\begin{figure}[!t]
    \centering
    \includegraphics[width=\linewidth]{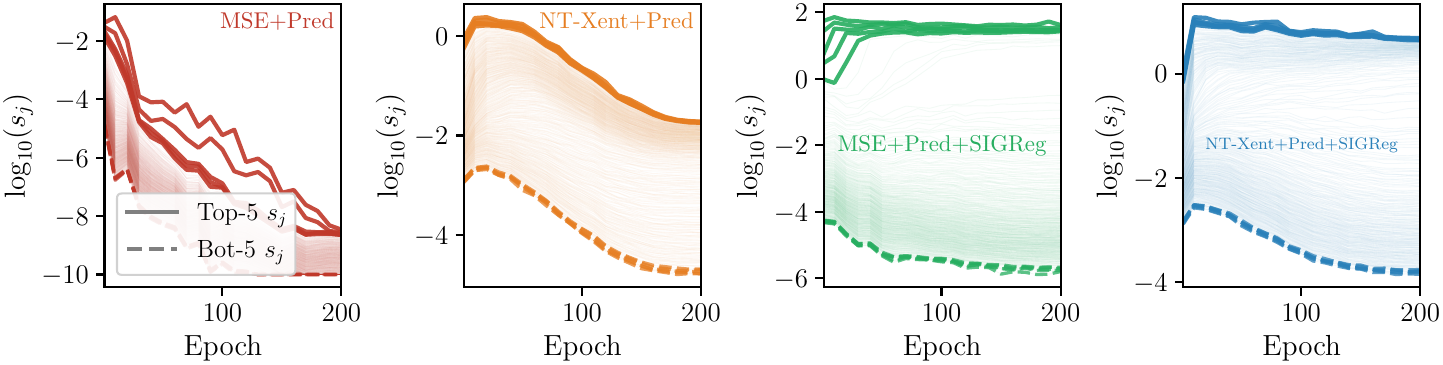}
    \caption{\textbf{Eigenspectrum of the projector-output covariance, verifying Theorem~\ref{theorem:collapse}.} Negative-free alignment with prediction but no regularization (MSE, no regularization) collapses to a few dominant eigenvalues, consistent with the optimization collapse of Theorem~\ref{theorem:collapse}. The contrastive variant (NT-Xent, no regularization) exhibits partial dimensional resistance, consistent with Remark~\ref{remark:ntxent_partial}, but does not reach the flat spectrum achieved by an explicit regularizer.}
    \label{fig:eigenspectrum}
\end{figure}

\begin{wrapfigure}{r}{0.4\textwidth}
    \centering
    \includegraphics[width=0.4\textwidth]{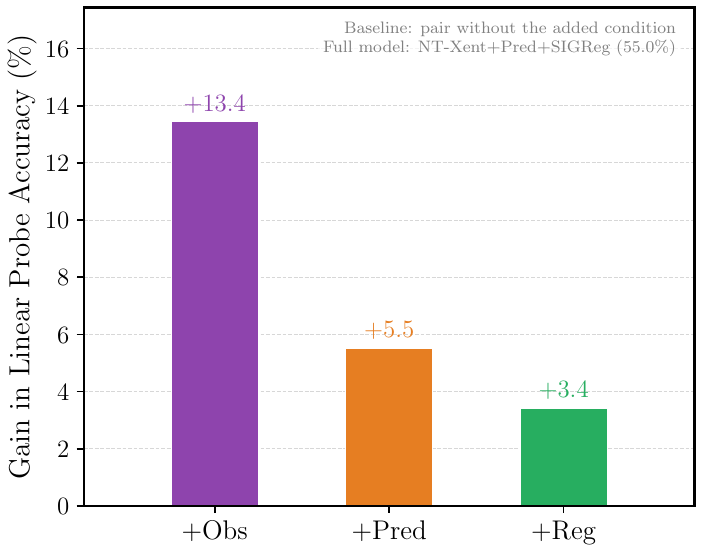}
    \caption{\textbf{Marginal gain in linear-probe accuracy from adding each principle, verifying Propositions~\ref{proposition:necessity_reg}, \ref{proposition:necessity_pred}, and \ref{proposition:necessity_obs}.} Each bar is the accuracy gained by adding one principle to the pair that omits it, relative to our proposed model (row~\textcolor{gray}{J}, $55.0\%$). No principle is redundant at this scale, and observation contributes the largest gain.}
    \label{fig:marginal_gain}
    \vspace{-15pt}
\end{wrapfigure}

\paragraph{\bf Collapse without regularization (Theorem~\ref{theorem:collapse}).} The direct signature of Theorem~\ref{theorem:collapse} is a degenerate representation spectrum (Figure~\ref{fig:eigenspectrum}). With the negative-free alignment and no regularizer, the projector-output covariance is dominated by a few leading eigenvalues (dimensional collapse), while adding an explicit regularizer flattens it. The contrastive alignment without a regularizer sits between the two, reflecting the self-limiting partial resistance of Remark~\ref{remark:ntxent_partial}: the NT-Xent term raises effective rank above the collapsed baseline but does not stabilize it to the level achieved by an explicit $\Lreg$. The linear-probe consequence is read from Table~\ref{tab:components}: the negative-free configuration without regularization (row~\textcolor{gray}{G}) reaches only $38.7\%$, over ten points below the regularized version (row~\textcolor{gray}{I}, $49.2\%$), while the contrastive variant (row~\textcolor{gray}{H}) retains $51.6\%$, the implicit resistance of the contrastive term standing in for an explicit regularizer at this scale. The same separation holds dynamically in effective rank (Appendix~\ref{appendix:effrank}).

\paragraph{\bf Necessity of each principle (Propositions~\ref{proposition:necessity_reg}, \ref{proposition:necessity_pred}, and \ref{proposition:necessity_obs}).} Against our proposed model (row~\textcolor{gray}{J}, $55.0\%$), removing regularization (row~\textcolor{gray}{H}) costs $3.4$ points, removing prediction (row~\textcolor{gray}{F}) $5.5$ points, and removing observation (row~\textcolor{gray}{E}) $13.4$ points, the largest single drop. Every single-principle configuration (rows~\textcolor{gray}{A}--\textcolor{gray}{D}) lands between $17.3\%$ and $42.1\%$; at this scale no pair substitutes for the third, and the ordering matches the structural argument: observation, which alone carries augmentation invariance, is the most costly to remove. Row~\textcolor{gray}{C} (prediction only) achieves $42.1\%$ despite low effective rank ($15.9$), illustrating the distinction between dimensional collapse and semantic impoverishment from Section~\ref{sec:introduction}: prediction alone yields non-trivial semantic content through latent target learning, but the low effective rank signals an impoverished representation geometry that limits transfer. The smallest gap (removing regularization) is consistent with Remark~\ref{remark:ntxent_partial}: contrastive repulsion partially compensates for absent explicit $\Lreg$ at this scale. Per-principle gains and linear-probe trajectories appear in Appendix~\ref{appendix:necessity}.

\paragraph{\bf Patch-level retrieval (Proposition~\ref{proposition:necessity_pred}).}
\label{subsec:retrieval}
To directly test the spatial consequence of the prediction principle, we evaluate patch-level nearest-neighbour retrieval. For each test image we extract the $144$ patch tokens from the frozen encoder, identify the $k$ nearest-neighbour patches in the training set by cosine similarity, and measure the proportion of retrieved patches falling within the same spatial quadrant as the query patch (\emph{spatial recall@k}). The results are shown in Table~\ref{tab:retrieval}. Removing prediction (row~\textcolor{gray}{F}) causes a $3.1$-point drop in spatial recall@5 relative to the full model (row~\textcolor{gray}{J}), directly confirming the structural argument of Proposition~\ref{proposition:necessity_pred}: without a spatial training signal the encoder does not learn to distinguish patch positions. Removing observation (row~\textcolor{gray}{E}) causes a smaller drop of $0.6$ points at recall@5, consistent with prediction being the primary driver of spatial structure and observation playing a complementary rather than conflicting role. The full model (row~\textcolor{gray}{J}) achieves the best spatial recall across all $k$.

\begin{table}[t]
    \small
    \centering
    \caption{\textbf{Patch-level nearest-neighbour retrieval, verifying Proposition~\ref{proposition:necessity_pred}.} Spatial recall@$k$ measures the proportion of retrieved patches falling within the same spatial quadrant as the query. Removing prediction (row~\textcolor{gray}{F}) causes the largest drop, confirming the structural argument that without $\Lpred$ the encoder receives no spatial training signal.}
    \label{tab:retrieval}
    \begin{tabular}{clllccc}
        \toprule
         & $\Linv$ & $\Lpred$ & $\Lreg$ & Recall@1 & Recall@5 & Recall@10 \\
        \midrule
        \textcolor{gray}{E} & none    & JEPA & SIGReg  & 48.2 & 92.4 & 98.2 \\
        \textcolor{gray}{F} & NT-Xent & none & SIGReg  & 43.9 & 89.9 & 98.0 \\
        \rowcolor{gray!12}
        \textcolor{gray}{J} & NT-Xent & JEPA & SIGReg  & \textbf{49.9} & \textbf{93.0} & \textbf{98.9} \\
        \bottomrule
    \end{tabular}
\end{table}

\paragraph{\bf Invariance objective and regularizer interaction.}
\label{subsec:mse_anomaly}
Row~\textcolor{gray}{I} achieves $49.2\%$ at effective rank $9.4$, far below row~\textcolor{gray}{J} ($87.5$) despite both having an explicit regularizer. We attribute this to a gradient-magnitude imbalance: the MSE loss collapses more aggressively than SIGReg can counteract at $\gamma = 0.05$. The regularizer is strictly positive at the collapsed solution (Lemma~\ref{lemma:sigreg_rdmreg_positive}) and prevents optimization collapse per Theorem~\ref{theorem:collapse}, but the effective rank trajectory (Appendix~\ref{appendix:effrank}) shows SIGReg stabilizing at low rank when paired with MSE. Increasing $\gamma$ raises effective rank toward the contrastive-regime values but at a cost to accuracy, consistent with SIGReg's isotropic-Gaussian target being a poor match for the natural representation geometry. We therefore read row~\textcolor{gray}{I} not as a failure of the regularization principle but as evidence that the invariance objective choice affects the regularizer's effective operating point; the contrastive term's implicit repulsion complements rather than substitutes for explicit regularization.

\paragraph{\bf Gradient complementarity (Theorem~\ref{theorem:gradient_decomp}).} Figure~\ref{fig:grad} measures the per-patch gradient norm of each objective at the encoder output across checkpoints. The invariance gradient is uniformly distributed across patch tokens (the contrastive objective acts on the pooled representation), and the predictive gradient is concentrated on context positions and vanishes at masked targets (the stop-gradient eliminates flow through the target branch). The contrast is stable across all checkpoints, confirming the two objectives act on disjoint aspects of the representation at the encoder output rather than competing; whether they remain aligned in parameter space (Remark~\ref{remark:param_conflict}) is a separate question we do not resolve here.

\begin{figure}[t]
    \centering
    \includegraphics[width=\linewidth]{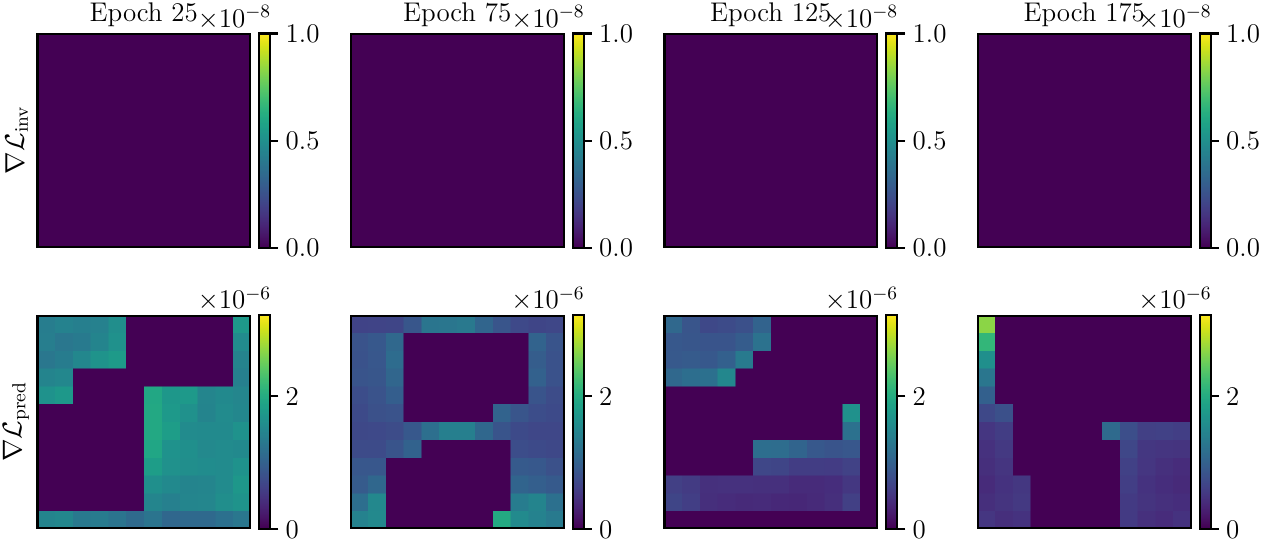}
    \caption{\textbf{Per-patch gradient norm at the encoder output, verifying Theorem~\ref{theorem:gradient_decomp}.} Top row: invariance gradient, uniform across patch tokens. Bottom row: predictive gradient, concentrated on visible context patches and vanishing at masked targets. Columns are training checkpoints; the contrast is stable throughout training.}
    \label{fig:grad}
\end{figure}

\paragraph{\bf Momentum encoder redundancy (Theorem~\ref{theorem:ema_collapse}).} Figure~\ref{fig:momentum}(a) confirms the convergence: the parameter distance between online and momentum encoders rises during early learning, then decays monotonically toward zero. The consequence is shown in Figure~\ref{fig:momentum}(b) and Table~\ref{tab:momentum}: the momentum encoder (row~\textcolor{gray}{B}) retains a modest edge over the stop-gradient baseline ($57.5\%$ versus $55.0\%$), comparable in size to the cost of removing regularization altogether, which we attribute to transient target smoothing rather than collapse prevention; removing the stop-gradient (row~\textcolor{gray}{C}) sends the model through a collapse trough early in training before the explicit regularizer drives full recovery to $55.9\%$. The shared asymptote confirms that the fixed point is set by the regularizer, while the target-encoder mechanism only stabilizes the path to it, consistent with Theorem~\ref{theorem:ema_collapse} and the transient analysis of Tian \etal~\cite{tian2021understanding}.

\begin{figure}[t]
    \centering
    \begin{subfigure}[b]{0.45\textwidth}
        \centering
        \includegraphics[width=\linewidth]{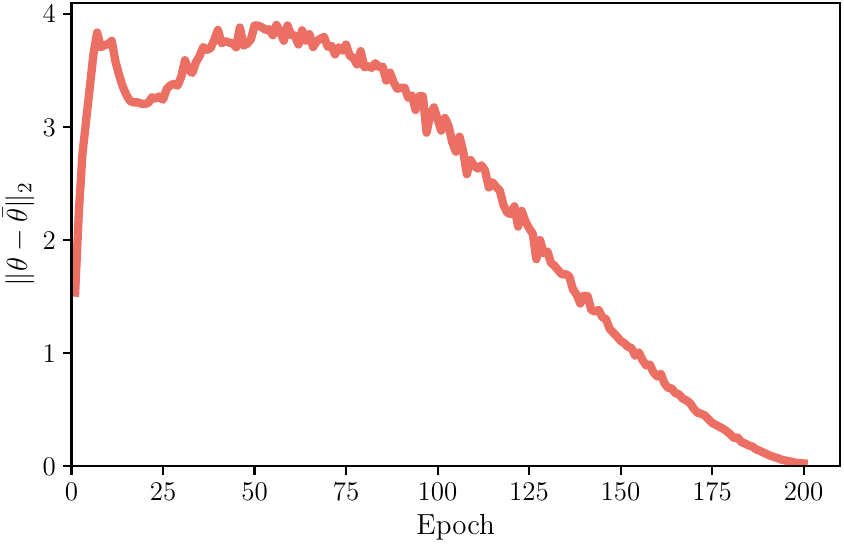}
        \caption{Online and mom. param.\ distance}
        \label{fig:ema_distance}
    \end{subfigure}
    \hspace{0.04\textwidth}
    \begin{subfigure}[b]{0.45\textwidth}
        \centering
        \includegraphics[width=\linewidth]{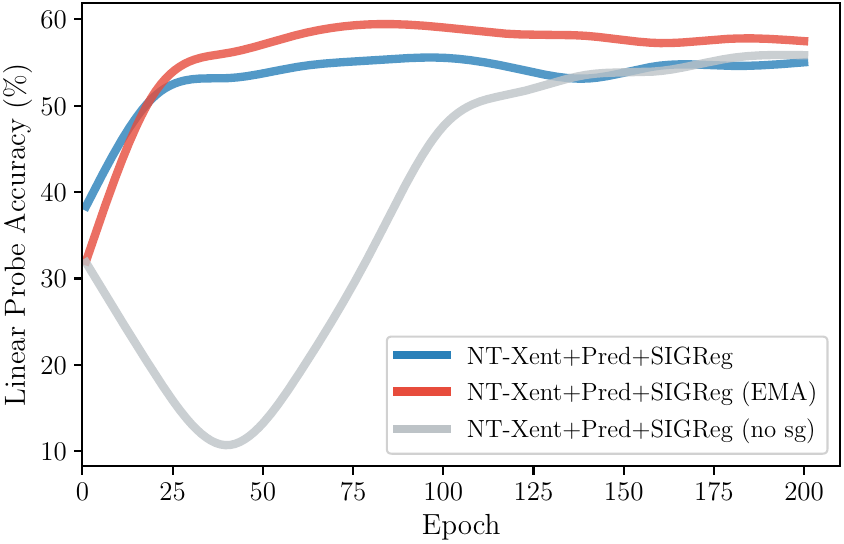}
        \caption{Linear probe, three target encoders}
        \label{fig:ema}
    \end{subfigure}
    \caption{\textbf{Momentum-encoder redundancy, verifying Theorem~\ref{theorem:ema_collapse}.} (a)~The online and momentum parameter distance rises during early learning, then decays toward zero; the two converge to the same fixed point, consistent with the theorem's convergence assumption. (b)~Stop-gradient and momentum reach comparable accuracy; removing the stop-gradient produces a collapse trough, after which the regularizer drives full recovery to the same asymptote.}
    \label{fig:momentum}
\end{figure}

\begin{table}[t]
    \small
    \centering
    \caption{\textbf{Target-encoder ablation showing momentum-encoder redundancy.} Contrastive alignment with isotropic-Gaussian regularization, verifying Theorem~\ref{theorem:ema_collapse}; all three target-encoder choices reach a comparable asymptote, with the momentum variant retaining a modest edge attributable to transient target smoothing rather than collapse prevention.}
    \label{tab:momentum}
    \begin{tabular}{clcc}
        \toprule
         & Target encoder & Effective Rank & Linear Probe \\
        \midrule
        \rowcolor{gray!12}
        \textcolor{gray}{A} & Stop-gradient        & 87.5 & 55.0 \\
        \textcolor{gray}{B} & Momentum             & 83.2 & \textbf{57.5} \\
        \textcolor{gray}{C} & None (full gradient) & 89.9 & 55.9 \\
        \bottomrule
    \end{tabular}
\end{table}

\paragraph{\bf Interchangeability of regularizers.} With prediction and contrastive alignment fixed, SIGReg and RDMReg are interchangeable ($55.0\%$ and $55.1\%$, training curves in Appendix~\ref{appendix:reg}), consistent with the contrastive-regularized duality of Garrido \etal~\cite{garrido2023duality}. VCReg lags by $\sim$4 points despite the highest effective rank; we attribute this to gradient conflict with $\Lpred$, since the covariance penalty targets off-diagonal correlations that the predictor actively induces in the encoder output. Negative-free alignment trails every contrastive variant by at least $5.8$ points, confirming that the contrastive term's implicit anti-collapse contribution is complementary to explicit regularization. Effective rank is a collapse diagnostic, not a quality score~\cite{garrido2023rankme}.

\begin{table}[t]
    \small
    \centering
    \caption{\textbf{Regularizer comparison with prediction enabled.} The isotropic-Gaussian and sliced-Wasserstein regularizers are interchangeable; the variance-covariance regularizer attains the highest effective rank but lower accuracy, attributable to gradient conflict with the prediction objective.}
    \label{tab:reg}
    \begin{tabular}{lllcc}
        \toprule
        $\Linv$ & $\Lpred$ & $\Lreg$ & Effective Rank & Linear Probe \\
        \midrule
        MSE     & JEPA & SIGReg  & 9.4   & 49.2 \\
        \rowcolor{gray!12}
        NT-Xent & JEPA & SIGReg  & 87.5  & 55.0 \\
        NT-Xent & JEPA & VCReg   & 248.4 & 51.1 \\
        NT-Xent & JEPA & RDMReg  & 109.8 & \textbf{55.1} \\
        \bottomrule
    \end{tabular}
\end{table}

\paragraph{\bf Discussion.} These are validation, not benchmark claims; ViT-Tiny on STL-10 is not comparable to large-scale pre-training, and linear-probe figures should be read only against one another. Within that scope every prediction is borne out: collapse as a degenerate spectrum, necessity as ordered accuracy and spatial recall drops, gradient complementarity in the per-patch gradients, and momentum redundancy as converging parameter distance. Finite-step optimization stops short of literal rank-one collapse, and the momentum variant keeps a modest edge the fixed-point analysis does not model. The measurements are structural: each principle contributes a separable, non-substitutable signal, and prior methods' implicit mechanisms reappear as counterparts of the explicit regularizer. Open questions: whether the ordering and regularizer interchangeability persist at scale; whether a dense probe would turn the structural argument for prediction into a measured result; and whether the gradient-conflict account of VCReg holds for other regularizers and architectures.

\section{Conclusion}
\label{sec:conclusion}

We formalized self-supervised representation learning as energy minimization over three non-overlapping principles: observation, prediction, and regularization. Regularization is optimization-theoretically necessary under negative-free alignment; contrastive resistance is self-limiting by gradient decay; prediction and observation are structurally necessary by construction. At the scale we study no pair substitutes for the third, as validated by linear-probe accuracy, patch retrieval, and gradient heatmaps. The projector and predictor realize observation and prediction architecturally; the two objectives are gradient-complementary at the encoder output under a mean-pool projector; and the momentum encoder converges to the online encoder's fixed point, offering no collapse guarantee beyond an explicit regularizer. The same decomposition recovers every major self-supervised method as a special case, differing only in which principles it satisfies. The clearest direction for future work is scale, and extending the decomposition to video and vision-language settings~\cite{bardes2024vjepa,assran2025vjepa2}.

\section*{Acknowledgements}

We acknowledge the computational resources and support provided by the Imperial College Research Computing Service (\url{http://doi.org/10.14469/hpc/2232}), which enabled our experiments.

%
%
\bibliographystyle{splncs04}
\bibliography{main}

\clearpage
\appendix
\section{Energy-Based Training Strategies}
\label{appendix:ebm}

We recall the two standard strategies for training an EBM. Making $F_w(\bx, \by)$ low for a training pair is straightforward; the difficulty is preventing the low-energy region from expanding to cover all of $\mathcal{Y}$.

\paragraph{\bf Contrastive training.} A contrastive objective pushes the energy of training pairs down while pushing the energy of contrastive samples up~\cite{lecun2022path}. A simple instance is the hinge loss $\mathbb{E}[ F_w(\bx,\by) + \max(0, m - F_w(\bx,\hat{\by}))]$ with margin $m > 0$; a more expressive instance is InfoNCE~\cite{oord2019cpc}, using $K$ negatives through a softmax ranking
\begin{equation}
    \min_w \; -\mathbb{E}\!\left[ \log \frac{e^{-F_w(\bx,\by)/\tau}}{e^{-F_w(\bx,\by)/\tau} + \sum_{k=1}^{K} e^{-F_w(\bx,\hat{\by}_k)/\tau}} \right],
    \label{eq:infonce}
\end{equation}
with temperature $\tau > 0$. Contrastive methods subsume Siamese networks~\cite{chopra2005learning}, MoCo~\cite{he2020momentum}, SimCLR~\cite{chen2020simclr}, and masked autoencoders~\cite{he2022mae}, differing only in how negatives are generated~\cite{lecun2022path,giakoumoglou2024rrd}.

\paragraph{\bf Regularized training.} Regularized methods avoid explicit negatives by minimizing the low-energy volume directly~\cite{lecun2022path}, adding a term $\Omega(F_w)$ that shrink-wraps the low-energy region around the data manifold: $\min_w \mathbb{E}[ F_w(\bx,\by)] + \Omega(F_w)$. Variance-covariance penalties~\cite{bardes2022vicreg, zbontar2021barlow}, whitening~\cite{ermolov2021whitening}, and distribution-matching objectives~\cite{balestriero2025lejepa, kuang2026lpjepa} impose qualitatively different geometric constraints.

\section{Architecture and Gradient Flows}
\label{appendix:architecture}

\Cref{fig:architecture} illustrates how the three principles of \Cref{eq:unified} map onto concrete architectural components and their respective gradient flows through the shared encoder $f_\theta$. The left branch realizes the observation principle (\seenoevil): both augmented views $\bx$ and $\by$ pass through $f_\theta$, their patch tokens are mean-pooled to $\sx$ and $\sy$, and the projector $h_\psi$ maps these to the space where $\Linv$ is computed. Because $h_\psi$ acts on the pooled token, the resulting gradient at the encoder output is uniform across all $N$ patch positions, $\frac{1}{N} \frac{\partial \Linv}{\partial \sx}$, as proved in \Cref{theorem:gradient_decomp}. The right branch realizes the prediction principle (\hearnoevil): context tokens $f_\theta(\bx)|_\mathcal{C}$ and mask token $m_t$ are fed to the predictor $g_\phi$, which regresses the target representation $f_\theta(\by)^{(t)}$; the stop-gradient $\sg(\cdot)$ blocks all gradient flow through the target branch, so $\Lpred$ drives the encoder only through the context tokens. The resulting gradient is spatially concentrated on visible context patches and zero at masked positions, the complementary local signal identified in \Cref{theorem:gradient_decomp}. The bottom strip realizes the regularization principle (\speaknoevil): the batch matrix $\mathbf{P} = [h_\psi(\sx); h_\psi(\sy)]$ is passed to $\Omega$, which produces a gradient that is strictly positive at every Dirac mass regardless of the state of $\Linv$ or $\Lpred$ (Lemmas~\ref{lemma:var_collapse} and~\ref{lemma:sigreg_rdmreg_positive}). The three gradient fields enter $f_\theta$ simultaneously and are structurally non-conflicting at the encoder output: the blue uniform field and the orange spatially-local field act on disjoint aspects of the representation, while the green batch-distribution field provides a collapse-prevention signal independent of both. Parameter-space alignment of the three fields is a separate empirical question addressed in \Cref{remark:param_conflict}.

\begin{figure}[h]
    \centering
    \includegraphics[width=\linewidth]{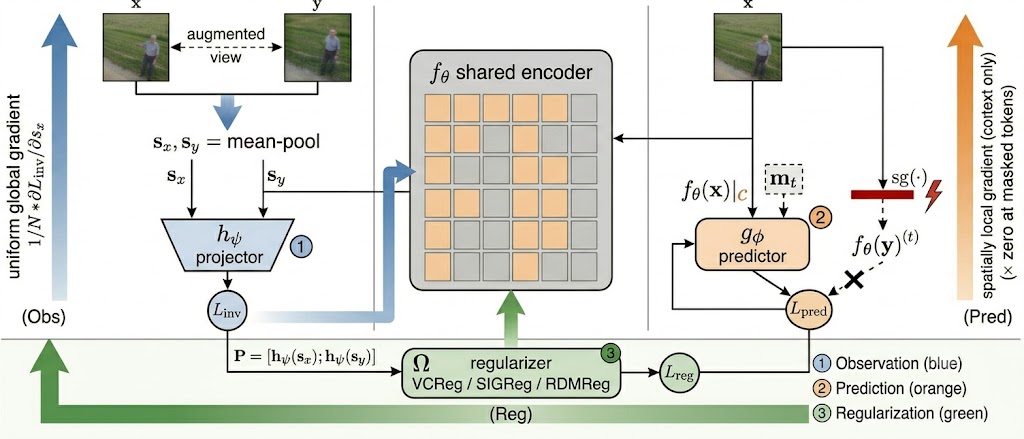}
    \caption{\textbf{Architecture and gradient flows for the three-principle decomposition.} The projector branch $h_\psi$ realizes the observation principle (\seenoevil), receiving a uniform global gradient from $\Linv$ through the mean-pooled representations $\sx, \sy$. The predictor $g_\phi$ realizes the prediction principle (\hearnoevil), receiving a spatially local gradient from $\Lpred$ concentrated on visible context patches; the stop-gradient blocks flow through the target branch. The regularizer $\Omega$ realizes the regularization principle (\speaknoevil), acting on the batch distribution of projected representations and providing a gradient at every step regardless of the invariance objective. The encoder $f_\theta$ (shared, gray) receives all three gradient fields, which are structurally non-conflicting at the encoder output (\Cref{theorem:gradient_decomp}). \textit{Figure generated with the assistance of Google Gemini, accessed June 16, 2026.}}
    \label{fig:architecture}
\end{figure}

\section{Full Taxonomy of Prior Methods}
\label{appendix:taxonomy}

Table~\ref{tab:taxonomy} lists every major self-supervised method recovered as a special case of the unified energy of \Cref{eq:unified}, organized by which of the three principles each satisfies. Methods marked $\dagger$ prevent collapse only through implicit mechanisms and therefore carry no collapse guarantee at convergence (Theorem~\ref{theorem:ema_collapse}). Methods marked $\ddagger$ realize prediction through a shared online tokenizer head rather than a dedicated predictor module. Methods satisfying the regularization principle (\cmark) but also retaining an implicit mechanism are marked $\dagger$ in the method name column only, to indicate the redundant implicit mechanism while preserving the correct \cmark\ in the Reg column; C-JEPA and DINOv2 fall in this category.

\begin{table}[ht]
    \small
    \centering
    \caption{All major self-supervised methods recovered as special cases of \Cref{eq:unified}, as stated in Proposition~\ref{proposition:unification}. Projector $h_\psi$ realizes Definition~\ref{def:obs}; predictor $g_\phi$ realizes Definition~\ref{def:pred}; $\Omega$ realizes Definition~\ref{def:reg}. $\dagger$~in the method name denotes a redundant implicit mechanism (momentum encoder or centering) alongside the explicit regularizer; $\dagger$~in the Reg column denotes implicit collapse prevention only; $\ddagger$~denotes prediction via a shared head rather than a dedicated predictor.}
    \label{tab:taxonomy}
    \resizebox{\textwidth}{!}{%
    \begin{tabular}{lccccccc}
        \toprule
        Method & Projector $h_\psi$ & $d$ & Predictor $g_\phi$ & $\Omega$ & Obs & Pred & Reg \\
        \midrule
        SimCLR~\cite{chen2020simclr}           & MLP       & InfoNCE       & none            & none    & \cmark & \xmark & \xmark$^\dagger$ \\
        MoCo~\cite{he2020momentum}             & MLP       & InfoNCE       & none            & none    & \cmark & \xmark & \xmark$^\dagger$ \\
        BYOL~\cite{grill2020byol}              & Asym. MLP & MSE           & none            & none    & \cmark & \xmark & \xmark$^\dagger$ \\
        SimSiam~\cite{chen2021simsiam}         & MLP       & neg. cos.     & none            & none    & \cmark & \xmark & \xmark$^\dagger$ \\
        DINO~\cite{caron2021dino}              & MLP       & CE (centered) & none            & none    & \cmark & \xmark & \xmark$^\dagger$ \\
        Barlow Twins~\cite{zbontar2021barlow}  & MLP       & cross-corr.   & none            & implicit & \cmark & \xmark & \cmark \\
        VICReg~\cite{bardes2022vicreg}         & MLP       & MSE           & none            & VCReg   & \cmark & \xmark & \cmark \\
        \midrule
        I-JEPA~\cite{assran2023ijepa}          & none      & none          & ViT             & none    & \xmark & \cmark & \xmark$^\dagger$ \\
        LeJEPA~\cite{balestriero2025lejepa}    & none      & none          & ViT             & SIGReg  & \xmark & \cmark & \cmark \\
        LpJEPA~\cite{kuang2026lpjepa}          & none      & none          & ViT             & RDMReg  & \xmark & \cmark & \cmark \\
        \midrule
        iBOT~\cite{zhou2022ibot}               & MLP       & CE (centered) & shared head$^\ddagger$ & none & \cmark & \cmark & \xmark$^\dagger$ \\
        C-JEPA$^\dagger$~\cite{mo2024cjepa}    & MLP       & MSE           & ViT             & VCReg   & \cmark & \cmark & \cmark \\
        DINOv2$^\dagger$~\cite{oquab2024dinov2}& MLP       & CE (centered) & shared head$^\ddagger$ & KoLeo & \cmark & \cmark & \cmark \\
        \midrule
        \rowcolor{gray!12}
        Ours & MLP & NT-Xent / MSE & ViT & Any & \cmark & \cmark & \cmark \\
        \bottomrule
    \end{tabular}}
\end{table}

\section{Regularization Objectives in Full}
\label{appendix:objectives}

Each regularizer operates on the batch matrix of projected representations $\mathbf{P} = [h_\psi(\sx); h_\psi(\sy)] \in \mathbb{R}^{2B \times D'}$.

\paragraph{\bf Variance-covariance regularization (VCReg).} The variance and covariance terms of VICReg~\cite{bardes2022vicreg} without the invariance term:
\begin{equation}
    \mathcal{L}_{\mathrm{VCReg}} = \mu\, \mathcal{L}_{\mathrm{var}}(\mathbf{P}) + \nu\, \mathcal{L}_{\mathrm{cov}}(\mathbf{P}),
    \label{eq:vcreg}
\end{equation}
\begin{equation}
    \mathcal{L}_{\mathrm{var}}(\mathbf{P}) = \frac{1}{D'} \sum_{d=1}^{D'} \max\!\left(0,\, \gamma_0 - \sqrt{\mathrm{Var}(\mathbf{P}^{(:,d)}) + \delta}\right), \qquad
    \mathcal{L}_{\mathrm{cov}}(\mathbf{P}) = \frac{1}{D'} \sum_{i \neq j} C_{ij}^2,
    \label{eq:var_cov}
\end{equation}
with $C$ the batch covariance of $\mathbf{P}$, $\gamma_0 = 1$, $\delta = 10^{-4}$, $\mu = 25$, $\nu = 1$.

\paragraph{\bf Isotropic-Gaussian regularization (SIGReg).} Matching the characteristic function of the batch distribution to a standard isotropic Gaussian via the Epps-Pulley statistic~\cite{balestriero2025lejepa}:
\begin{equation}
    \mathcal{L}_{\mathrm{SIGReg}} = \mathbb{E}_{\boldsymbol{\omega}}\!\left[ \left| \mathbb{E}_{\mathbf{p}}[\cos(\boldsymbol{\omega}^\top \mathbf{p})] - e^{-\|\boldsymbol{\omega}\|^2/2} \right|^2 + \left| \mathbb{E}_{\mathbf{p}}[\sin(\boldsymbol{\omega}^\top \mathbf{p})] \right|^2 \right],
    \label{eq:sigreg}
\end{equation}
with $\boldsymbol{\omega}$ sampled on the unit sphere using $K = 17$ quadrature points.

\paragraph{\bf Sliced-Wasserstein regularization (RDMReg).} Aligning the batch distribution to a rectified generalized Gaussian target $\mathbf{X}_{\mathrm{RGG}}$ via sliced Wasserstein distance~\cite{kuang2026lpjepa}:
\begin{equation}
    \mathcal{L}_{\mathrm{RDMReg}} = \mathbb{E}_{\boldsymbol{\theta}}\!\left[ \mathcal{W}_2^2\!\left( \mathrm{sort}(\mathbf{P}^\top \boldsymbol{\theta}),\; \mathrm{sort}(\mathbf{X}_{\mathrm{RGG}}^\top \boldsymbol{\theta}) \right) \right],
    \label{eq:rdmreg}
\end{equation}
over $128$ random one-dimensional projections, with target shape parameter $p = 1$, zero mean, and unit variance after rectification.

\section{Hyperparameters}
\label{appendix:hyperparams}

\paragraph{\bf Pretraining.} We pre-train on STL-10 with a ViT-Tiny encoder (patch size 8, embedding dimension 192, depth 12, three heads) at $96 \times 96$ resolution, yielding $144$ patch tokens. The predictor is a narrow Vision Transformer (embedding dimension 96, depth 4) and the projector maps the pooled representation through width 1024 to output dimension 256. Each image yields two views; masking samples one context block at scale $[0.85, 1.0]$ and four target blocks at scale $[0.15, 0.2]$. Training runs for 200 epochs (10 warmup) with AdamW at batch size 512, a peak learning rate of $5 \times 10^{-4}$ decayed cosinewise to $1 \times 10^{-6}$, weight decay annealed from $0.04$ to $0.4$, and bfloat16 precision. The loss weights are $\beta = 1.0$ for prediction, $\alpha = 0.1$ for NT-Xent ($1.0$ for MSE), and $\gamma = 0.05$, $0.02$, $1.0$ for SIGReg, VCReg, and RDMReg respectively. Remaining settings are the NT-Xent temperature $\tau = 0.07$, $K = 17$ SIGReg quadrature points, $128$ RDMReg projections, VCReg coefficients $\mu = 25$ and $\nu = 1$, and a momentum coefficient $m = 0.996$ where applicable. Every experiment is run over five random seeds.

\paragraph{\bf Linear probing.} We freeze the pre-trained encoder and train a single linear head on its mean-pooled representations for 30 epochs with AdamW at batch size 512, reporting top-1 accuracy on the STL-10 test split. No augmentation other than standard normalization is applied at probe time, and the encoder receives no gradient.

\paragraph{\bf Patch-level retrieval.} We extract patch tokens from the frozen encoder for all test images and a random subset of $10{,}000$ training images. For each of the $144$ query patches per test image, we retrieve the $k$ nearest training patches by cosine similarity and measure spatial recall@$k$: the proportion of retrieved patches whose spatial quadrant (top-left, top-right, bottom-left, bottom-right) matches the query patch quadrant. We report $k \in \{1, 5, 10\}$ and average over all query patches and test images.

\paragraph{\bf Effective rank.} We report effective rank as a scalar measure of representational non-degeneracy. Given the covariance $\Sigma$ of the projector outputs over a batch with eigenvalues $\{\lambda_i\}_{i=1}^{D'}$ and normalized spectrum $p_i = \lambda_i / \sum_j \lambda_j$, the effective rank is $\exp\!\left(-\sum_i p_i \log p_i\right)$, the exponential of the Shannon entropy of the normalized eigenspectrum. A value near $1$ indicates collapse onto a single direction, while a value approaching $D'$ indicates an isotropic, fully utilized representation space. Effective rank is a collapse diagnostic, not a downstream quality predictor; high effective rank is necessary but not sufficient for high linear-probe accuracy, as illustrated by Table~\ref{tab:reg}.

\section{Full Proofs}
\label{appendix:proofs}

This section gives the proofs deferred from the main text, in the order they are referenced: the collapse theorem, gradient complementarity, the strict positivity of each regularizer at the collapsed solution, and momentum-encoder redundancy.

\subsection{Full Proof of Theorem~\ref{theorem:collapse}}
\label{appendix:collapse}

\begin{proof}[Full proof of Theorem~\ref{theorem:collapse}]
Let $f_\theta(\bx) = \mathbf{c}\,\mathbf{1}_N^\top$ for all $\bx$. Then $\sx = \sy = \mathbf{c}$, so $h_\psi(\sx) = h_\psi(\sy)$, giving $\mathcal{L}_{\mathrm{MSE}} = \frac{1}{B}\sum_{n} \|h_\psi(\sx^{(n)}) - h_\psi(\sy^{(n)})\|^2 = 0$. For the predictive loss, $f_\theta(\by)^{(t)} = \mathbf{c}$ for all $n, t$. The stop-gradient on $f_\theta(\by)^{(t)}$ blocks gradient flow from the target representations back through the encoder, so the collapsed solution is achievable via the context branch alone: a predictor with zero attention weights and output bias $\mathbf{c}$ satisfies $g_\phi(\mathbf{c}\,\mathbf{1}_{|\mathcal{C}|}^\top, m_t) = \mathbf{c}$, giving $\Lpred = 0$. Both terms are non-negative and zero, so the collapsed solution is a global minimizer for any $\alpha, \beta > 0$.
\end{proof}

\subsection{Full Proof of Theorem~\ref{theorem:gradient_decomp}}
\label{appendix:grad}

\begin{proof}[Full proof of Theorem~\ref{theorem:gradient_decomp}]
We prove the uniform projection for both $\mathcal{L}_{\mathrm{MSE}}$ and $\mathcal{L}_{\mathrm{NT\text{-}Xent}}$ under a mean-pool projector architecture, where $h_\psi$ acts on $\sx = \frac{1}{N}\sum_i f_\theta(\bx)^{(i)}$.

\emph{Gradient of $\mathcal{L}_{\mathrm{MSE}}$.} The loss depends on each patch token $f_\theta(\bx)^{(i)}$ only through $\sx$. By the chain rule, $\partial \mathcal{L}_{\mathrm{MSE}} / \partial f_\theta(\bx)^{(i)} = (\partial \mathcal{L}_{\mathrm{MSE}} / \partial \sx) \cdot (\partial \sx / \partial f_\theta(\bx)^{(i)}) = (\partial \mathcal{L}_{\mathrm{MSE}} / \partial \sx) \cdot \frac{1}{N}\mathbf{I}_D$, which is uniform across all $i$.

\emph{Gradient of $\mathcal{L}_{\mathrm{NT\text{-}Xent}}$.} The NT-Xent loss operates on $h_\psi(\sx^{(n)})$, which depends on $f_\theta(\bx)^{(i)}$ only through $\sx^{(n)}$ (via mean pooling followed by the MLP projector $h_\psi$). The same chain rule gives $\partial \mathcal{L}_{\mathrm{NT\text{-}Xent}} / \partial f_\theta(\bx)^{(i)} = (\partial \mathcal{L}_{\mathrm{NT\text{-}Xent}} / \partial \sx) \cdot \frac{1}{N}\mathbf{I}_D$, also uniform. This argument holds for any projector architecture where $h_\psi$ acts on the pooled token $\sx$ rather than on individual patch tokens.

\emph{Gradient of $\Lpred$.} $\Lpred$ receives only context tokens $f_\theta(\bx)|_\mathcal{C}$ as input to $g_\phi$; masked tokens $i \notin \mathcal{C}$ do not appear in the computational graph of $\Lpred$. The stop-gradient on $\syn{n}{t} = f_\theta(\by)^{(t)}$ blocks all flow through the target branch. Therefore $\partial \Lpred / \partial f_\theta(\bx)^{(i)} = 0$ for all $i \notin \mathcal{C}$. In expectation over random masking, the predictive gradient is non-zero for every $i$ but spatially non-uniform, its magnitude depending on local prediction difficulty. The combined gradient field from $\Linv$ (uniform global) and $\Lpred$ (spatially local, context-concentrated) is structurally non-conflicting at the encoder output. We emphasize that this disjointness is stated at the encoder output; the two fields still share the encoder parameters $\theta$, so parameter-space alignment (Remark~\ref{remark:param_conflict}) is a separate, empirical matter.
\end{proof}

\subsection{Regularization Prevents the Collapsed Solution}
\label{appendix:reg_prevents}

\begin{lemma}{VCReg is strictly positive at the collapsed solution}{var_collapse}
If $f_\theta$ is collapsed ($\bs_{y,n} = \mathbf{c}$ for all $n$), then $\mathcal{L}_{\mathrm{VCReg}} \geq \mu(\gamma_0 - \sqrt{\delta}) > 0$ for any $\mu > 0$ and $\gamma_0 > \sqrt{\delta}$.
\end{lemma}
\begin{proof}
At collapse $\mathrm{Var}(\mathbf{P}^{(:,d)}) = 0$ for every $d$, so from \Cref{eq:var_cov} $\mathcal{L}_{\mathrm{var}}(\mathbf{P}) = \gamma_0 - \sqrt{\delta} > 0$ since $\gamma_0 = 1 > 10^{-2} = \sqrt{\delta}$, giving $\mathcal{L}_{\mathrm{VCReg}} \geq \mu(\gamma_0 - \sqrt{\delta}) > 0$.
\end{proof}

\begin{theorem}{VCReg prevents the collapsed global minimizer}{vcreg_prevents}
For any $\gamma > 0$, the collapsed solution is not a global minimizer of $\Ltotal = \alpha\mathcal{L}_{\mathrm{MSE}} + \beta\Lpred + \gamma\mathcal{L}_{\mathrm{VCReg}}$.
\end{theorem}
\begin{proof}
By Theorem~\ref{theorem:collapse} and Lemma~\ref{lemma:var_collapse} the collapsed solution achieves $\mathcal{L}_{\mathrm{MSE}} = \Lpred = 0$ and $\mathcal{L}_{\mathrm{VCReg}} \geq \mu(\gamma_0 - \sqrt{\delta}) > 0$, so $\Ltotal|_{\mathrm{collapsed}} \geq \gamma\mu(\gamma_0 - \sqrt{\delta}) > 0$. Any encoder whose projections satisfy $\mathrm{Var}(\mathbf{P}^{(:,d)}) \geq \gamma_0^2$ with zero off-diagonal covariance (achievable by centering and scaling per dimension) has $\mathcal{L}_{\mathrm{VCReg}} = 0$, so $\Ltotal = \alpha\mathcal{L}_{\mathrm{MSE}} + \beta\Lpred$, and each term can be driven toward zero by training. Hence $\Ltotal$ can be made strictly below $\gamma\mu(\gamma_0 - \sqrt{\delta})$, and the collapsed solution is not a global minimizer.
\end{proof}

\begin{lemma}{SIGReg and RDMReg are strictly positive at the collapsed solution}{sigreg_rdmreg_positive}
If $f_\theta$ is collapsed, the empirical batch distribution is a Dirac mass $\delta_{\mathbf{c}'}$, and $\mathcal{L}_{\mathrm{SIGReg}} > 0$, $\mathcal{L}_{\mathrm{RDMReg}} > 0$.
\end{lemma}
\begin{proof}
For SIGReg, the characteristic function of $\delta_{\mathbf{c}'}$ at any $\boldsymbol{\omega} \neq \mathbf{0}$ has modulus one and deviates strictly from the Gaussian target $e^{-\|\boldsymbol{\omega}\|^2/2} < 1$; since the quadrature includes such a point, $\mathcal{L}_{\mathrm{SIGReg}} > 0$. For RDMReg, the sliced Wasserstein distance between a Dirac mass and a continuous target is strictly positive. In both cases Theorem~\ref{theorem:vcreg_prevents} applies with $\mathcal{L}_{\mathrm{reg}}$ replaced accordingly.
\end{proof}

\begin{remark}{Consistency with Definition~\ref{def:reg}}{reg_consistency}
Lemmas~\ref{lemma:var_collapse} and~\ref{lemma:sigreg_rdmreg_positive} verify that all three regularizers are strictly positive at every Dirac mass, satisfying the first clause of Definition~\ref{def:reg}. The minimizers of each regularizer are full-rank over $\mathbb{R}^D$: SIGReg is uniquely minimized by the standard isotropic Gaussian, RDMReg by the rectified generalized Gaussian target $\mathbf{X}_{\mathrm{RGG}}$, and VCReg by any per-dimension variance-floored, decorrelated distribution, all of which are full-rank. Hence each studied regularizer satisfies Definition~\ref{def:reg}.
\end{remark}

\subsection{Full Proof of Theorem~\ref{theorem:ema_collapse}}
\label{appendix:ema}

\begin{proof}[Full proof of Theorem~\ref{theorem:ema_collapse}]
\emph{Assumption:} $\theta(t) \to \theta^*$ as $t \to \infty$. The continuous-time momentum update $\dot{\bar\theta} = (1-m)(\theta - \bar\theta)$ has solution $\bar\theta(t) = e^{-(1-m)t}\bar\theta(0) + (1-m)\int_0^t e^{-(1-m)(t-s)}\theta(s)\,ds$. For any $\epsilon > 0$ choose $T$ with $\|\theta(s) - \theta^*\| \leq \epsilon$ for $s \geq T$; then for $t > T$, $\|\bar\theta(t) - \theta^*\| \leq e^{-(1-m)t}\|\bar\theta(0) - \theta^*\| + e^{-(1-m)(t-T)} M + \epsilon$ with $M = \sup_s \|\theta(s) - \theta^*\| < \infty$. As $t \to \infty$ both exponentials vanish, so $\|\bar\theta(t) - \theta^*\| \leq \epsilon$; since $\epsilon$ is arbitrary, $\bar\theta(t) \to \theta^*$. If $\theta^*$ is collapsed the momentum target converges to that state, whereas Theorem~\ref{theorem:vcreg_prevents} removes the collapsed state as a fixed point for any $\gamma > 0$. We note that the assumption of convergence of $\theta(t)$ simplifies the true coupled dynamics, which involve stop-gradient and are more complex; the theorem characterizes fixed points rather than training trajectories.
\end{proof}

\section{Batch Size and the Contrastive Gradient}
\label{appendix:batchsize}

We establish here that $\Lpred$ and $\Lreg$ maintain informative gradients independently of batch size, complementing the discussion of Remark~\ref{remark:ntxent_partial} on the self-limiting nature of the contrastive gradient.

$\Lpred$ (\Cref{eq:pred}) is a sum over samples and target patches, each term depending only on its own sample. The regularizers ($\mathcal{L}_{\mathrm{SIGReg}}$, $\mathcal{L}_{\mathrm{VCReg}}$, $\mathcal{L}_{\mathrm{RDMReg}}$) depend on within-batch statistics but require no paired negatives and maintain a strictly positive gradient at the collapsed solution for any $B \geq 2$ (Lemmas~\ref{lemma:var_collapse} and~\ref{lemma:sigreg_rdmreg_positive}). By contrast, the NT-Xent gradient scales with the number of in-batch negatives $K = B - 1$; as $B$ decreases, the implicit anti-collapse signal of Remark~\ref{remark:ntxent_partial} weakens proportionally, making explicit $\Lreg$ increasingly critical at small batch sizes. Table~\ref{tab:batchsize} shows linear-probe accuracy at $B \in \{128, 256, 512\}$ for the contrastive-only (row~\textcolor{gray}{H}) and full-model (row~\textcolor{gray}{J}) configurations.

\begin{table}[ht]
    \small
    \centering
    \caption{\textbf{Batch-size sensitivity of contrastive-only versus full-model configurations.} Without explicit regularization (row~\textcolor{gray}{H}), accuracy degrades more steeply as batch size decreases because the contrastive anti-collapse signal weakens with fewer negatives. The full model (row~\textcolor{gray}{J}) is more robust due to $\Lreg$ maintaining a batch-size-independent gradient.}
    \label{tab:batchsize}
    \begin{tabular}{cllccc}
        \toprule
        & $\Linv$ & $\Lreg$ & $B=128$ & $B=256$ & $B=512$ \\
        \midrule
        \textcolor{gray}{H} & NT-Xent & none   & 44.2 & 48.1 & 51.6 \\
        \rowcolor{gray!12}
        \textcolor{gray}{J} & NT-Xent & SIGReg & \textbf{52.8} & \textbf{53.7} & \textbf{55.0} \\
        \bottomrule
    \end{tabular}
\end{table}

\section{Supplementary Experiment: Effective Rank Dynamics}
\label{appendix:effrank}

Figure~\ref{fig:effrank} traces effective rank across training for the collapse experiment, the dynamic counterpart of Figure~\ref{fig:eigenspectrum}, as across-seed means over five seeds. Negative-free runs without regularization remain low-rank; contrastive runs begin at higher rank through implicit repulsion but do not stabilize on their own; the explicit regularizer holds effective rank in a healthy band, confirming Theorem~\ref{theorem:collapse} and Remark~\ref{remark:ntxent_partial}. The MSE + JEPA + SIGReg run (row~\textcolor{gray}{I}) stabilizes at low effective rank despite the regularizer, consistent with the gradient-magnitude imbalance discussed in Section~\ref{subsec:mse_anomaly}: the SIGReg gradient is insufficient to overcome the strong collapsing pull of the MSE term at the chosen hyperparameter settings.

\begin{figure}[ht]
    \centering
    \includegraphics[width=0.55\linewidth]{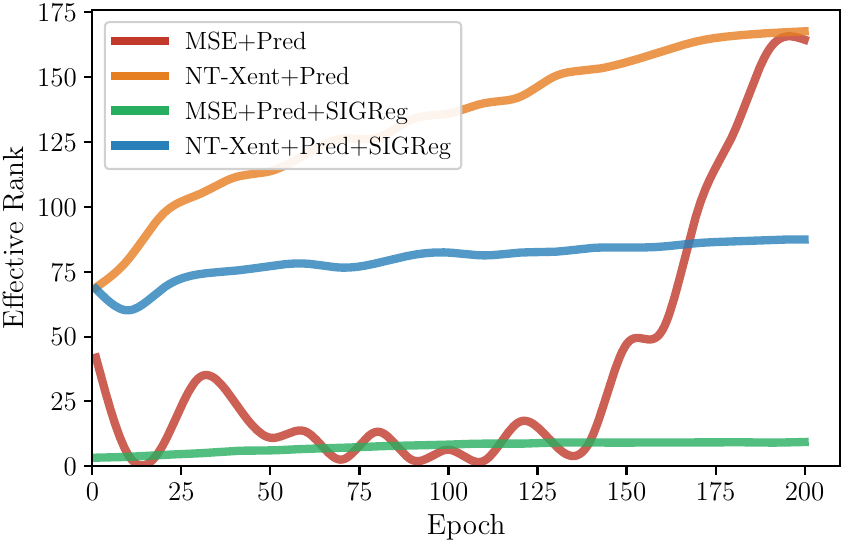}
    \caption{\textbf{Effective rank over training.} Negative-free alignment without regularization stays low-rank; the contrastive term raises rank but does not stabilize it; the explicit regularizer holds it in a healthy band. The MSE + SIGReg run (row~\textcolor{gray}{I}) stabilizes at low rank, reflecting a gradient-magnitude imbalance between MSE and SIGReg at the chosen hyperparameter settings. Curves are means over five seeds.}
    \label{fig:effrank}
\end{figure}

\section{Supplementary Experiment: Necessity of Each Principle}
\label{appendix:necessity}

Figure~\ref{fig:marginal_gain} summarizes the marginal contribution of each principle, the accuracy gained by adding it to the pair that omits it; no principle is redundant at the studied scale, and observation contributes the largest gain. Figure~\ref{fig:linprobe} traces linear-probe accuracy across training for every strict subset, the dynamic counterpart, with the all-three configuration dominating throughout and the single-principle runs clustering well below. All curves are across-seed means over five seeds.

\begin{figure}[ht]
    \centering
    \includegraphics[width=0.55\linewidth]{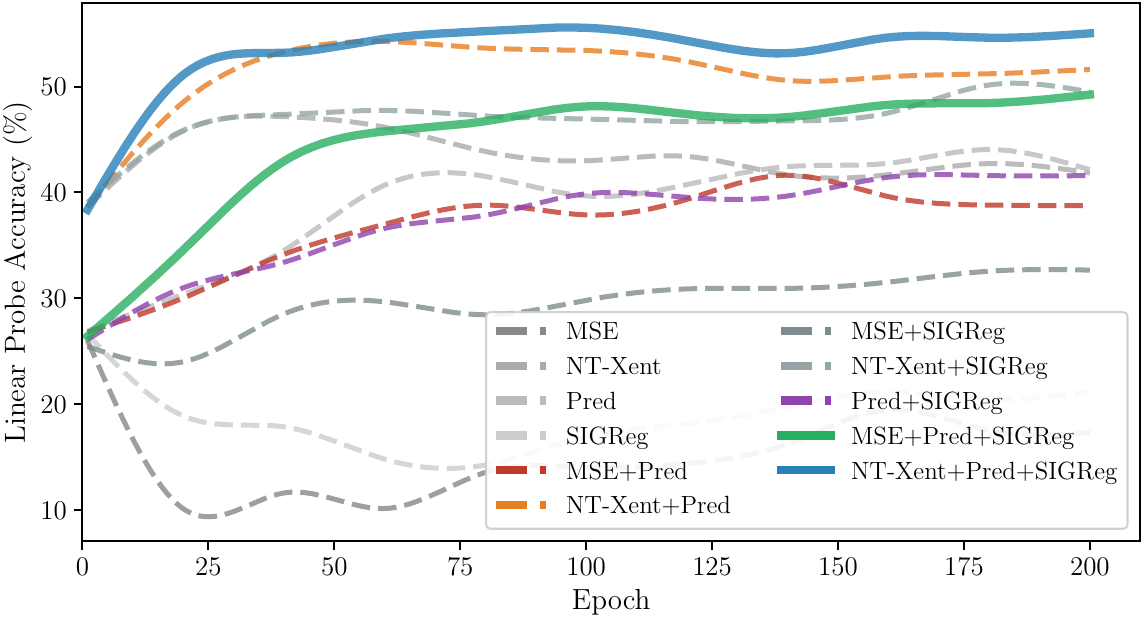}
    \caption{\textbf{Linear-probe accuracy over training for every strict subset.} The all-three configuration (top) dominates; single-principle runs cluster well below. Curves are means over five seeds.}
    \label{fig:linprobe}
\end{figure}

\section{Supplementary Experiment: Regularizer Curves}
\label{appendix:reg}

Figure~\ref{fig:reg} shows the linear-probe trajectories for the four regularizer choices, as across-seed means over five seeds. The isotropic-Gaussian and sliced-Wasserstein regularizers converge to the same accuracy; the variance-covariance regularizer lags despite the highest effective rank, consistent with the gradient-conflict hypothesis discussed in Section~\ref{sec:experiments}; the negative-free alignment trails every contrastive variant.

\begin{figure}[ht]
    \centering
    \includegraphics[width=0.55\linewidth]{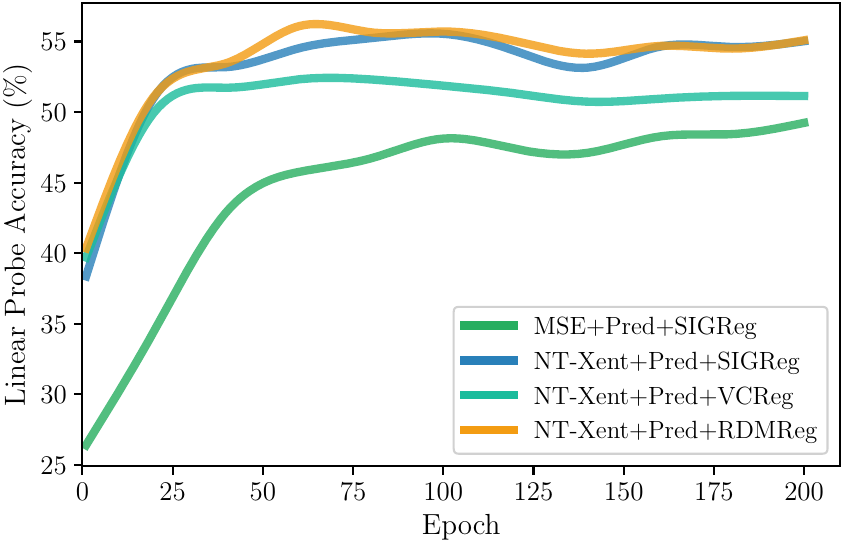}
    \caption{\textbf{Linear-probe accuracy for four regularizer choices.} The isotropic-Gaussian and sliced-Wasserstein regularizers are interchangeable; the variance-covariance regularizer lags despite the highest effective rank; the negative-free alignment trails every contrastive variant. Curves are means over five seeds.}
    \label{fig:reg}
\end{figure}

\section{Broader Impact}

Our contribution is analytical rather than a new state-of-the-art system: we identify which parts of a self-supervised objective are doing which work, and show that collapse prevention can be isolated in an explicit geometric term rather than entangled with the primary training signal. If this decomposition holds at scale, it offers a design discipline for future self-supervised algorithms, in which observation, prediction, and regularization can be specified, swapped, and diagnosed independently instead of being rediscovered implicitly through architectural tricks such as momentum encoders, centering, or carefully tuned asymmetries. This modularity also lowers the cost of principled ablation, since a practitioner can attribute a failure to a missing principle rather than to an opaque interaction, and it suggests a route to extending the same accounting to video and vision-language pre-training, where the corresponding objectives are less well understood. The broader risks are those shared by representation learning generally: encoders trained on uncurated data inherit the biases of that data, and a framework that makes such training cheaper or easier to tune propagates those biases more widely. Our validation is at a single small scale, so the practical claims should be read as hypotheses about larger models rather than guarantees, and we would caution against treating the effective-rank diagnostics we report as a proxy for representation quality in deployment settings.

\end{document}